%% file: main.tex
\documentclass[letterpaper]{article} 
\usepackage{aaai2027}  
\usepackage[hyphens]{url}  
\usepackage{graphicx} 
\usepackage{natbib}  
\usepackage{caption} 
\usepackage{algorithm}
\usepackage{algorithmic}
\usepackage{subcaption}

\usepackage{amsmath}
\usepackage{amsfonts}
\usepackage[notextcomp]{stix}

\newcommand{\rightrelation}[2]{%
  \mathrel{\underset{#2}{\overset{#1}{\rightsquigarrow}}}
}

\newcommand{\leftrightrelation}[2]{%
  \mathrel{\underset{#2}{\overset{#1}{\leftrightsquigarrow}}}
}

\newcommand{\Ber}[0]{\text{Ber}}
\newcommand{\MM}[0]{\text{MM}}

\newcommand{\ImS}[0]{\text{Im}(S)}
\newcommand{\Supp}[1]{\text{Supp}(#1)}

\usepackage{amsthm}
\theoremstyle{definition}
\newtheorem{theorem}{Theorem}
\newtheorem{proposition}{Proposition}

\newtheorem{corollary}{Corollary}
\newtheorem{definition}{Definition}

\usepackage{newfloat}
\usepackage{listings}
\DeclareCaptionStyle{ruled}{labelfont=normalfont,labelsep=colon,strut=off} 
\floatstyle{ruled}
\newfloat{listing}{tb}{lst}{}
\floatname{listing}{Listing}

\usepackage{booktabs}

\nocopyright

\title{Proportional Analogies on Probability Distributions via Bayesian Updating}
\author{
    Pierre-Alexandre Murena
}
\affiliations{
    Hamburg University of Technology (TUHH)\\
    Research Group on Human-Centric Machine Learning \\
    pierre-alexandre.murena@tuhh.de
}

\begin{document}

\maketitle

\begin{abstract}

\input{sections/abstract}
\end{abstract}

\begin{links}
    \link{Code}{https://github.com/ppaamm/Analogies-on-Probability-Distributions-by-Bayes-Updating}
\end{links}

\section{Introduction}
\input{sections/introduction}

\section{Background and Related Works}
\input{sections/related}

\section{Bayes-Based Analogies}

\input{sections/bayes-analogies}

\section{Bayes-Based Analogies on the Exponential Family}

\input{sections/exponential-family}

\section{Instantiation of Classical Exponential Families}
\input{sections/examples-exponential-family}

\section{Approximate Solver of Analogical Equations}
\input{sections/approximate-solver}



\section{Conclusion}
\label{sec:conclusion}
\input{sections/conclusion}

\section*{Acknowledgments}

The author would like to warmly thank Jean Lieber for the supporting discussions and feedback about this work. 

\bibliography{aaai2027}

\end{document}

%% file: sections/abstract.tex
Analogies are quaternary relations of the form ``A is to B as C is to D". Among the various formalizations of analogical reasoning, proportional analogies provide an important axiomatic framework by characterizing valid analogies through a set of postulates. While proportional analogies have been extensively studied over Boolean, symbolic, and real-valued domains, their extension to probability distributions remains largely unexplored. In this paper, we introduce a notion of proportional analogy for probability distributions based on Bayesian updating. Our approach builds upon the idea that two distributions are related whenever one can be transformed into the other through Bayesian updating induced by a suitable set of observations. We investigate this framework for several standard members of the exponential family and discuss how it naturally extends to arbitrary probability distributions through Gaussian mixture approximations.

%% file: sections/introduction.tex
Analogical reasoning is a fundamental aspect of human cognition~\cite{mitchell2021abstraction}, enabling the comparison of objects through their structural similarities~\cite{gentner1983structure}. An analogy is defined as a quaternary relation of the form ``$a$ is to $b$ as $c$ is to $d$'', usually denoted by $a : b :: c : d$. Although analogy has traditionally been studied within symbolic artificial intelligence, it has recently attracted increasing attention in numerical machine learning. Notable applications include the characterization of word embeddings~\cite{mikolov2013efficient} and Large Language Models~\cite{webb2023emergent}, as well as data augmentation~\cite{couceiro2017analogy} and the study of inductive reasoning~\cite{murena2022measuring,bayoudh2007learning}.

Despite these advances, the study of analogical reasoning remains largely confined to a limited range of data domains. Existing work has mainly focused on hierarchical structures~\cite{gick1980analogical}, Boolean data~\cite{prade2017boolean}, character strings~\cite{lepage1998solving}, and vector spaces~\cite{rumelhart1973model}. In contrast, analogies between probability distributions have received comparatively little attention, despite their potential importance for machine learning. To the best of our knowledge, only two approaches have been proposed. \citet{murena2018opening} define analogies through parallel transport in information geometry. Their approach is restricted to distributions belonging to the same parametric family, is computationally demanding, and does not satisfy accepted principles of analogical reasoning. More recently, \citet{prade2025analogical} proposed an extension of classical proportional analogies to categorical distributions. However, their framework is limited to this specific family, has a restricted domain of definition, and does not explicitly exploit the statistical nature of probability distributions.

In this paper, we introduce a new notion of analogy between probability distributions based on Bayesian updating. The central idea is to characterize the similarity between two distributions through the existence of a dataset capable of transforming one distribution into the other by Bayesian inference with a given likelihood. This perspective naturally grounds analogical reasoning in the fundamental mechanisms of probabilistic inference while being applicable to arbitrary probability distributions. Moreover, our framework satisfies the classical properties expected from analogical reasoning and admits an elegant characterization for important statistical families.  Our contributions are the following: 
\begin{enumerate}
    \item We introduce a novel definition of proportional analogy between probability distributions based on Bayesian updating, and show that it satisfies the fundamental properties of proportional analogies;
    \item We prove that, for distributions belonging to the exponential family, the proposed notion admits a simple arithmetic characterization in the natural parameter space;
    \item We develop a sampling-based algorithm for solving analogical equations between arbitrary probability distributions, making the proposed framework applicable beyond analytically tractable models.
\end{enumerate}

%% file: sections/related.tex
There exists various approaches to analogies, which all build upon the idea of correspondence between two domains, based on their inherent structure and of the transfer of properties~\cite{hall1989computational}. Several computational approaches were proposed for analogies. 

\subsection{Computational Approaches to Analogy}

The \emph{model-based} vision consists in using a model to describe the underlying structure of the domains and enforce the analogies. The Structure-Mapping Theory~\cite{gentner1983structure} defines analogy as a maximal alignment of latent structures. \citet{antic2022analogical} characterizes analogies via universal algebra, where justifications are given as operations on shared primitives. Similarly, \citet{murena2022measuring} introduces models as programs on a Turing machine and define the alignment of domains by the existence of minimal inputs producing a description of the analogy. 

A special case of the model-based approach is the \emph{functional approach}, that states that $a : b :: c : d$ if and only if there exists a function $f$ such as $b = f(a)$ and $d = f(c)$. A modern example of such an approach is given by the work of \citet{lee2024enhancing} who build $f$ as the composition of elementary operations learned with reinforcement learning.

The \emph{axiomatic vision} describes analogies through a set of postulates that they must satisfy. The most common definition is the notion of \emph{proportional analogies}~\cite{lepage2004analogy}, also called \emph{analogical proportions} by some authors~\cite{prade2021analogical}.

\subsection{Proportional Analogies}

The definition of proportional analogies is a formalization of principles introduced by Aristotle, in the form of the following three principles:

\begin{definition}
A proportional analogy is a quaternary relation $\mathcal{A}$ such that for all $a, b, c, d$, we have:
\begin{enumerate}
    \item Reflexivity: $\mathcal{A}(a, b, a, b)$. 
    \item Symmetry: $\mathcal{A}(a, b, c, d)$ iff $\mathcal{A}(c, d, a, b)$
    \item Central permutation: $\mathcal{A}(a, b, c, d)$ iff $\mathcal{A}(a, c, b, d)$. 
\end{enumerate}
\label{def:proportional-analogy}
\end{definition}
Some authors also include a fourth postulate, called \emph{determinism}, and stating that $\mathcal{A}(a, b, a, x)$ implies $x = b$~\cite{miclet2008analogical}. An analogical proportion $\mathcal{A}(a, b, c, d)$ is often denoted $a : b :: c : d$. 

Analogical proportions have been extensively studied in the Boolean domain~\cite{prade2017boolean,antic2022analogical}, on character strings~\cite{lepage1998solving,stroppa2005analogical}, with applications e.g. in image processing~\cite{duck2022analogy}, recommender systems~\cite{hug2019analogical} or natural language~\cite{lepage2004analogy}. We refer the reader to the position paper of~\citet{prade2021analogical} for a more extensive presentation of applications of proportional analogies.

An important relation satisfying the properties of Definition~\ref{def:proportional-analogy} is the arithmetic analogy:
\begin{definition}
Four vectors $a, b, c, d \in \mathbb{R}^n$ are said to be in \emph{arithmetic analogy} is and only if $b - a = d - c$. 
\label{def:arithmetic-analogy}
\end{definition}

This analogy was introduced before the formalization of proportional analogies~\cite{rumelhart1973model} and is sometimes designated as \emph{parallelogram rule}.

\subsection{Proportional Analogies on Distributions}

Despite their potential application to machine learning, analogies on probability distributions have not been extensively studied. 

A first attempt was proposed by~\citet{murena2018opening}. The authors propose to extend the parallelogram rule to Riemannian manifolds, by interpreting subtractions as a logarithmic map and the addition as a parallel transport along a geodesic. The authors show that this definition fails to satisfy the properties of proportional analogies when the manifold is not flat. However, they show that their definition can apply to probability distributions in the Fisher-Rao manifold~\cite{nielsen2020elementary}, and they propose an algorithm for analogies on multivariate normal distributions. 

More recently, \citet{prade2025analogical} propose a direct application of Definition~\ref{def:proportional-analogy} to categorical distributions by the introduction of arithmetico-geometric analogies, satisfying arithmetic analogy (Definition~\ref{def:arithmetic-analogy}) to the probabilities themselves and to their logarithm. This approach has severe limitations. The double constraint strongly limits the sets of valid analogies, making it a too strong definition for practical applications. More importantly, their definition does not extend naturally to distributions beyond categorical distributions, since it does not rely on an understanding of the specific nature of probability distributions.

%% file: sections/bayes-analogies.tex
In this section, we introduce the concept of Bayes-based proportional analogies on the set of probability distributions. 

\subsection{Bayesian Inference}

Let $(\mathcal X,\mathcal A)$ be the observation space and $(\Theta,\mathcal T)$ the parameter space, both endowed with their respective $\sigma$-algebras. A statistical model is a family of probability measures $\mathcal L=\{P_\theta : \theta\in\Theta\}$ on $(\mathcal X,\mathcal A)$. When the measures $P_\theta$ are dominated by a common measure, we denote their densities by $p(x\mid\theta)$.

In the Bayesian framework, the parameter $\theta$ is modeled as a $\Theta$-valued random variable. A prior distribution on $\theta$ with density $p(\theta)$ specifies the uncertainty about $\theta$. Thus, the prior turns the measurable space $(\Theta,\mathcal T)$ into a probability space $(\Theta,\mathcal{T},\Pi)$. The prior distribution is updated, after observation of $x \in \mathcal{X}$, into the posterior distribution according to Bayes' rule:
\begin{equation}
    p(\theta\mid x) = \frac{p(x\mid\theta)\,p(\theta)}{p(x)},
    \label{eqn:bayes-rule}
\end{equation}
where $p(x) = \int_\Theta p(x\mid\vartheta)\,p(\vartheta)\,d\vartheta$ is the marginal likelihood. 

To account for the absence of observation, we extend the observation space by introducing a distinguished symbol $\varnothing \notin \mathcal X$ and define $\bar{\mathcal X} = \mathcal X \cup \lbrace \varnothing \rbrace$.
The posterior associated with the null observation is defined by $p(\theta \mid \varnothing) = p(\theta)$. Equivalently, Bayesian updating with $\varnothing$ acts as the identity on the space of prior distributions.

We use the notation $p \rightrelation{x}{\mathcal{L}} q$ to denote that $q$ is the posterior distribution obtained from the prior distribution $p$ after observing $x$ under the statistical model $\mathcal L$. Note that the relation $\rightrelation{x}{\mathcal{L}}$ is not symmetrical. We consider the symmetrized version by defining $p \leftrightrelation{x}{\mathcal{L}} q$ if and only if $p \rightrelation{x}{\mathcal{L}} q$ or  $q \rightrelation{x}{\mathcal{L}} p$

\subsection{Bayes-Based Proportional Analogies}

Consider a space $\mathcal{P}$ of probability distributions over $\Theta$ and a statistical model $\mathcal{L}$ with parameter in $\Theta$. We define the quaternary relation $R_{\mathcal{L}}$ over $\mathcal{P}$ as: 
\[
    R(p_A, p_B, p_C, p_D) \quad \text{iff \quad there exist } x_1, x_2 \in \bar{\mathcal{X}}; \begin{cases}
        p_A \leftrightrelation{x_1}{\mathcal{L}} p_B \\
        p_C \leftrightrelation{x_1}{\mathcal{L}} p_D \\
        p_A \leftrightrelation{x_2}{\mathcal{L}} p_C \\
        p_B \leftrightrelation{x_2}{\mathcal{L}} p_D
    \end{cases} 
\]

\begin{proposition}
If $\mathcal{P}$ and $\mathcal{L}$ are such that for all $p, q \in \mathcal{P}$, there exists $x \in \mathcal{X}$ such that $p \leftrightrelation{x}{\mathcal{L}} q$, then the relation $R$ defined above is a proportional analogy, and we call it \textit{Bayes-based proportional analogy for likelihood $\mathcal{L}$}. 
\label{prop:definition-BBAn}
\end{proposition}
\begin{proof}
The symmetry and central permutation postulates follow directly from the definition. To prove that $p_A : p_B ::^\mathcal{L} p_A : p_B$, we observe that, for any distribution $p \in \mathcal{P}$, we have $p \leftrightrelation{\varnothing}{\mathcal{L}} p$. The reflexivity follows directly from this observation and from the premises of the proposition. 
\end{proof}

A more restrictive definition requires that the four updates witnessing the analogy are shared across the corresponding pairs of distributions. That is, there exist $x_1, x_2, x_3, x_4 \in \bar{\mathcal{X}}$ such that $x_1$ induces both $p_A \rightsquigarrow p_B$ and $p_C \rightsquigarrow p_D$, $x_2$ induces both $p_B \rightsquigarrow p_A$ and $p_D \rightsquigarrow p_C$, $x_3$ induces both $p_A \rightsquigarrow p_C$ and $p_B \rightsquigarrow p_D$, and $x_4$ induces both $p_C \rightsquigarrow p_A$ and $p_D \rightsquigarrow p_B$. This amounts to defining a transformation between two distributions whenever one can be obtained from the other as either a prior-to-posterior or a posterior-to-prior update. This definition was not retained in this paper because it is overly restrictive. Indeed, Bayesian updates are generally not reversible. For example, in the family of normal distributions, every update decreases the posterior variance. Hence, if $\mathcal{N}_1 \rightsquigarrow \mathcal{N}_2$, the reverse transformation $\mathcal{N}_2 \rightsquigarrow \mathcal{N}_1$ cannot occur unless $\mathcal{N}_1 = \mathcal{N}_2$. As a result, this definition would preclude any non-trivial analogy between normal distributions. 

Our proposed definition of analogy between distributions admits a natural interpretation in terms of statistical reachability. A similar perspective is already implicit in standard notions of analogy. For instance, an arithmetic analogy (Definition~\ref{def:arithmetic-analogy}) can be characterized through a transformation $a \rightsquigarrow b$ induced by a translation.

This definition of analogy has an important consequence in terms of the distributions. 

\begin{proposition}
If $p_A : p_B ::_{\mathcal{L}} p_C : p_D$, then for all $\theta \in \Theta$, either (1) the distributions $\log p(\theta)$ are in arithmetic analogy; or (2) the reverse solution holds, i.e. $\log p_A(\theta) = \log p_D(\theta)$ and $\log p_B(\theta)=\log p_C(\theta)$.
\label{prop:characterization-bayes-analogies}
\end{proposition}
\begin{proof}
$p_A \rightrelation{x}{\mathcal{L}}$ is equivalent to:
\[
\log p_B(\theta) = \log p_A(\theta) + \log \mathcal{L}(x \mid \theta) - \log p(x). 
\]
Let $\theta \in \Theta$. There exists $x, x^\prime \in \mathcal{X}$ and $\sigma_{AB}$, $\sigma_{CD}$, $\sigma_{AC}$ and $\sigma_{BD} \in \lbrace 0, 1 \rbrace$ such that:
\begin{align*}
(-1)^{\sigma_{AB}} \big(\log p_B(\theta) - \log p_A(\theta) \big) = \log \mathcal{L}(x \mid \theta) - \log p(x) \\
(-1)^{\sigma_{CD}} \big(\log p_D(\theta) - \log p_C(\theta) \big) = \log \mathcal{L}(x \mid \theta) - \log p(x) \\
(-1)^{\sigma_{AC}} \big(\log p_C(\theta) - \log p_A(\theta) \big) = \log \mathcal{L}(x^\prime \mid \theta) - \log p(x^\prime) \\
(-1)^{\sigma_{BD}} \big(\log p_D(\theta) - \log p_B(\theta) \big) = \log \mathcal{L}(x^\prime \mid \theta) - \log p(x^\prime)
\end{align*}
which involves:
\[
\begin{cases}
\log p_B - \log p_A = (-1)^{\sigma_{CD} - \sigma_{AB}} \Big( \log p_D - \log p_C  \Big) \\
\log p_C - \log p_A = (-1)^{\sigma_{BD} - \sigma_{AC}} \Big( \log p_D - \log p_B  \Big) 
\end{cases}
\]
where we use the notation $\log p_i$ for $\log p_i(\theta)$.

In the case where $\sigma_{AB} = \sigma_{CD}$ and $\sigma_{AC} = \sigma_{BD}$, we obtain that $\log p_B(\theta) - \log p_A(\theta) = \log p_D(\theta) - \log p_C(\theta)$. 

In the case where $\sigma_{CD} = \sigma_{AB}$ and $|\sigma_{BD} - \sigma_{AC}| = 1$, the conditions imply that $\log p_A(\theta) = \log p_C(\theta)$ and $\log p_B(\theta) = \log p_D(\theta)$. Similarly, when $|\sigma_{CD} - \sigma_{AB}| = 1$ and $\sigma_{BD} = \sigma_{AC}$, we have $\log p_A(\theta) = \log p_B(\theta)$ and $\log p_C(\theta) = \log p_D(\theta)$. 

Finally, when $|\sigma_{BD} - \sigma_{AC}| = |\sigma_{CD} - \sigma_{AB}| = 1$, we observe $\log p_A(\theta) = \log p_D(\theta)$ and $\log p_B(\theta) = \log p_B(\theta)$. 
\end{proof}

The reverse solution ($\log p_A(\theta) = \log p_D(\theta)$ and $\log p_B(\theta)=\log p_C(\theta)$) is an important consequence of the choice of the $\leftrightsquigarrow$ relation. It corresponds to a case similar to the Klein model in analogies on Boolean~\cite{klein1981culture}, judged as debatable by some authors~\cite{prade2021analogical} while being naturally covered by some theoretical frameworks~\cite{antic2022analogical}.

\subsection{Analogical Equations}

Given three distributions $p_A, p_B$ and $p_C \in \mathcal{P}$, solving the analogical equation ``$p_A : p_B ::^{\mathcal{L}} p_C : q$'' consists in finding the distributions $q \in \mathcal{P}$ such that the analogy holds. We define the \emph{support of the analogical equation associated with $\mathcal{L}$} the set $\Supp{\mathcal{L}} \subseteq \mathcal{P}^3$ defined as the set of distributions $(p_A, p_B, p_C)$ such that the analogical equation $p_A : p_B ::^{\mathcal{L}} p_C : q$ admits at least one solution. Studying this set provides a natural way of evaluating the expressive power of a Bayes-based analogy: the larger the support, the larger the class of analogical problems that can be solved.

%% file: sections/exponential-family.tex
A probability belongs to the \emph{exponential family}~\cite{brown1986fundamentals} if its probability density function or probability mass function takes the form:
\begin{equation}
    p(\theta \mid \xi) = h(\theta) \exp \left( \eta(\xi)^T T(\theta) - A(\xi) \right)
    \label{eqn:exp-family}
\end{equation}
where $\xi$ is the parameter of the distribution, $h(\theta)$ the base measure, $T(\theta)$ the sufficient statistic, $\eta(\xi)$ the natural parameter of the distribution and $A(\xi)$ the log-partition. 

\subsection{Exponential Family as Conjugate Prior}

Consider the likelihood family $\mathcal{L}^S$ as the set of distributions of the form:
\begin{equation}
    p(x \mid \theta) \propto \exp(S(x)^T T(\theta))
\label{eqn:likelihood-exp-family-general}
\end{equation}
where $S(x)$ corresponds to the sufficient statistic of the observations with respect to the likelihood, that couples to the prior's sufficient statistic $T(\theta)$. 

\begin{proposition}
The exponential family is a conjugate prior of $\mathcal{L}^S$, and $p(\theta \mid \xi) \rightrelation{x}{\mathcal{L}^S} p(\theta \mid \xi^\prime)$ if and only if:
\begin{equation}
    \eta(\xi^\prime) = \eta(\xi) + S(x)
\end{equation}
\label{prop:conjugate-exponential}
\end{proposition}

\begin{corollary}
For all $x \in \mathcal{X}$, we have $p(\theta \mid \xi) \leftrightrelation{x}{\mathcal{L}^S} p(\theta \mid \xi^\prime)$ if and only if there exists $\sigma \in \lbrace 0, 1 \rbrace$ such that:
\begin{equation}
    (-1)^\sigma \Big( \eta(\xi^\prime) - \eta(\xi) \Big) = S(x)
\end{equation}
\label{cor:relation-condition-exp-fam}
\end{corollary}
\begin{proof}
Using Bayes rule:
\begin{align*}
    p(\theta \mid x, \xi) &\propto p(x \mid \theta) \, p(\theta \mid \xi) \\
    & \propto \exp\big(S(x)^T T(\theta)\big) \, h(x) \, \exp \big( \eta(\xi)^T T(\theta) - A(\xi) \big) \\
    &\propto h(x) \exp \big( (\eta(\xi) + S(x))^T T(\theta) - A(\xi) \big)
\end{align*}
which proves the proposition. 
\end{proof}

In the corollary, the quantity $\sigma$ indicates the direction of the relation: $\sigma = 0$ corresponds to  $p(\theta \mid \xi) \rightrelation{x}{\mathcal{L}^S} p(\theta \mid \xi^\prime)$ and $\sigma = 1$ to  $p(\theta \mid \xi^\prime) \rightrelation{x}{\mathcal{L}^S} p(\theta \mid \xi)$. 

Note that, in this context, the exponential family is used as the prior, and not as the likelihood, as traditionally discussed in the literature~\cite{diaconis1979conjugate}.

\subsection{Characterization of Bayes-Based Analogies on the Exponential Family}

The characterization of Bayes-based analogies on the exponential family follows directly from Proposition~\ref{prop:definition-BBAn} and Corollary~\ref{cor:relation-condition-exp-fam}.

\begin{proposition}
For all hyperparameters $\xi_A, \xi_B, \xi_C$ and $\xi_D$, if
\[
p(. \mid \xi_A) : p(. \mid \xi_B) ::^{\mathcal{L}^S} p(. \mid \xi_C) : p(. \mid \xi_D)
\]
then either (1) the natural parameters $\eta(\xi)$ satisfy the arithmetic analogy; or (2) the reverse solution holds, i.e. $\eta(\xi_A)=\eta(\xi_D)$ and $\eta(\xi_B)=\eta(\xi_C)$.
\label{prop:characterization-ExpFam}
\end{proposition}

\begin{proof}
The proposition follows directly from Corollary~\ref{cor:relation-condition-exp-fam} and Proposition~\ref{prop:characterization-bayes-analogies}. 
\end{proof}


The necessary condition of the arithmetic proportion between natural parameters raises an important issue: the condition depends on the choice of the natural representation~$\eta$. The choice of a natural representation is not unique for a same distribution. In general, we can define an equivalence relation~$\sim$ between tuples $(T, \eta)$ and $(T^\prime, \eta^\prime)$ when there exists functions $\theta \mapsto u(\theta)$ and $\xi \mapsto v(\xi)$ such that:
\begin{equation}
    \eta(\xi)^T T(\theta) = \eta^\prime(\xi)^T T^\prime(\theta) + u(\theta) + v(\xi).
\end{equation}

With such a definition of equivalence, it is possible to extend a given sufficient statistic $T(\theta)$ with some constants, making the dimension of the sufficient statistics arbitrarily large. \citet{brown1986fundamentals} defines the minimality of a sufficient statistics $T(.)$ if they are affinely independent.

An important result is the independence to the choice of the minimal natural representation: 

\begin{proposition}
Let $(T, \eta)$ and $(T^\prime, \eta^\prime)$ such that $(T, \eta) \sim (T^\prime, \eta^\prime)$ and $T^\prime$ is a minimal sufficient statistic. If $\eta(\xi_A)$, $\eta(\xi_B)$, $\eta(\xi_C)$ and $\eta(\xi_D)$ are in arithmetic analogy, then $\eta^\prime(\xi_A)$, $\eta^\prime(\xi_B)$, $\eta^\prime(\xi_C)$ and $\eta^\prime(\xi_D)$ are also in arithmetic analogy. 
\end{proposition}

\begin{proof}
Multiplying the arithmetic condition on $\eta(\xi)$ by $T(\theta)$ on the right, we obtain for all $\theta$:
\[
(\eta(\xi_B) - \eta(\xi_A) + \eta(\xi_C) - \eta(\xi_D))^T T(\theta) = 0
\]
Since $(T, \eta) \sim (T^\prime, \eta^\prime)$, this rewrites as:
\begin{align*}
&(\eta^\prime(\xi_B) - \eta^\prime(\xi_A) + \eta^\prime(\xi_C) - \eta^\prime(\xi_D))^T T^\prime(\theta) \\
&\qquad + v(\xi_B) - v(\xi_A) + v(\xi_C) - v(\xi_D) = 0
\end{align*}
We observe that the term on the second line is a constant w.r.t. $\theta$. 
Since $T^\prime$ is minimal, it is affinely independent, which guarantees $\eta^\prime(\xi_B) - \eta^\prime(\xi_A) + \eta^\prime(\xi_C) - \eta^\prime(\xi_D) = 0$. 
\end{proof}

Based on the last two propositions, we see that analogy is preserved by change of minimal sufficient statistics, and, consequently, independent of the chosen representation of the distribution. Based on this result, we know assume that we work exclusively with minimal statistics.

\subsection{Defining Valid Analogies}

The result of Proposition~\ref{prop:characterization-ExpFam} provides a necessary condition for distributions to be in analogy. We saw that it does not depend on the choice of the function $S(x)$ in the likelihood family. This condition is not necessarily sufficient: it assumes the existence of observations $x, x^\prime \in \mathcal{X}$ that satisfy the conditions of Corollary~\ref{cor:relation-condition-exp-fam}, which is not always guaranteed. In order to introduce these, we consider the image of $S$, written $\ImS$, and defined as the set $\ImS = \lbrace S(x) : x \in \bar{\mathcal{X}} \rbrace$. Using the result of Proposition~\ref{prop:characterization-ExpFam}, we can express a full characterization of Bayes-based analogies in an exponential family using the natural parameters.

\begin{theorem}
For any $\xi_A, \xi_B, \xi_C, \xi_D \in \Xi$, we have
\[
p(. \mid \xi_A) : p(. \mid \xi_B) ::^{\mathcal{L}^S} p(. \mid \xi_C) : p(. \mid \xi_D)
\]
if and only if $\lbrace \eta(\xi_A) - \eta(\xi_B), \eta(\xi_B) - \eta(\xi_A) \rbrace \cap \ImS$ and $\lbrace \eta(\xi_C) - \eta(\xi_A), \eta(\xi_A) - \eta(\xi_C) \rbrace \cap \ImS$ are not empty, and either of the two following conditions is satisfied:
\begin{enumerate}
    \item $\eta(\xi_A) = \eta(\xi_D)$, $\eta(\xi_B) = \eta(\xi_C)$; or 
    \item $\eta(\xi_B) - \eta(\xi_A) = \eta(\xi_D) - \eta(\xi_C)$. 
\end{enumerate}
\label{thm:valid-analogies-expfamily}
\end{theorem}

\begin{proof}
Considering that the analogy holds, we can deduce from Proposition~\ref{prop:characterization-ExpFam} that either $\eta_A = \eta_D$ and $\eta_B = \eta_C$, or that $\eta_B - \eta_A = \eta_D - \eta_C$ (using the simplified notation $\eta_i$ for $\eta(\xi_i)$). The conditions on the sets follows directly from the definition. 

Assuming that $\lbrace \eta(\xi_A) - \eta(\xi_B), \eta(\xi_B) - \eta(\xi_A) \rbrace \cap \ImS$ and $\lbrace \eta(\xi_C) - \eta(\xi_A), \eta(\xi_A) - \eta(\xi_C) \rbrace \cap \ImS$ are not empty, there exists $\sigma_{AB}, \sigma_{AC} \in \lbrace 0, 1 \rbrace$ and $x, x^\prime \in \bar{\mathcal{X}}$ such that $(-1)^{\sigma_{AB}} (\eta_B - \eta_A) = S(x)$ and $(-1)^{\sigma_{AC}} (\eta_C - \eta_A) = S(x^\prime)$.  

If $\eta_D = \eta_A$ and $\eta_C = \eta_B$, we have $\eta_D - \eta_C = \eta_A - \eta_B = (-1)^{1+ \sigma_{AB}} S(x)$. Additionally, $\eta_C - \eta_A = \eta_B - \eta_A = (-1)^{\sigma_{AB}} S(x)$ and
$\eta_D - \eta_B = \eta_A - \eta_B = (-1)^{\sigma_{AB} + 1} S(x)$. By definition, this shows that the analogy holds.

If $\eta_B - \eta_A = \eta_D - \eta_C$, we have $\eta_D - \eta_C = (-1)^{\sigma_{AB}} S(x)$ and $\eta_D - \eta_B = \eta_C - \eta_A = (-1)^{\sigma_{AC}} S(x^\prime)$. By definition, the analogy holds.
\end{proof}

In this theorem, the sufficient statistic $S$ governs the existence of solutions through its image. Hence, the expressive power of the analogy relation is directly tied to the geometry of $\ImS$: richer images allow a larger set of valid analogies.

\begin{proposition}
Given two likelihoods $\mathcal{L}_S$ and $\mathcal{L}_{S^\prime}$, if $\ImS \subseteq \text{Im}(S^\prime)$, then four distributions in analogy with respect to $\mathcal{L}_S$ are also in analogy with respect to $\mathcal{L}_{S^\prime}$. 
\end{proposition}

As a direct consequence of this proposition, we can define a partial order between analogies, induced by the choices of $S$. This order corresponds to the inclusion of analogies seen as subsets of $\mathcal{P}^
4$. We call \emph{maximal} an analogy associated to a likelihood $S_{\max}$ such that $\ImS \subseteq \text{Im}(S_{\max})$ for all $S$. 




\subsection{Divergence-Based Characterization of Analogies}

An alternative, coordinate-free characterization of Bayes-based analogies can be established using the Kullback-Leibler divergence. Because the KL divergence in a minimal exponential family equals the Bregman divergence generated by the log-partition function $A(\eta)$, the linear arithmetic structure maps directly to information geometry.

\begin{proposition}
Denote $M(p, q) = \arg\min_r (D_{\text{KL}}( p\| r) + D_{\text{KL}}( q\| r))$. If four distributions $p_A, p_B, p_C$ and $p_D$ of the exponential family are in analogy, then either $M(p_A, p_D) = M(p_B, p_C)$ or $M(p_A, p_C) = M(p_B, p_D)$.
\end{proposition}

\begin{proof}
For a minimal sufficient statistic $\eta$, we can verify that $\eta(M(p, q)) = \frac{1}{2}(\eta(p) + \eta(q))$. We have the analogy if $\eta(p_A) = \eta(p_D)$ and $\eta(p_B) = \eta(p_C)$, in which case $M(p_A, p_C) = M(p_B, p_D)$, or if $\eta(p_B) - \eta(p_A) = \eta(p_D) - \eta(p_C)$, in which case $M(p_A, p_D) = M(p_B, p_C)$. 
\end{proof}

The converse implications do not hold in general: equality of the corresponding barycenters entails equality of the natural-parameter displacements, but does not ensure that their common value belongs to $\ImS$.

This result states that every valid analogy induces a common KL barycenter for one of these two pairings of distributions. This characterization by the barycenters is a well-known property of arithmetic analogies in $\mathbb{R}^d$, used for instance by \citet{lepage2024any} as a definition of analogies.

\subsection{Analogical Equations}

Theorem~\ref{thm:valid-analogies-expfamily} characterizes the quadruples of distributions that are in analogy. The problem of defining the support $\Supp{\mathcal{L}_S}$ of analogical equations on an exponential family is closely related. 

\begin{theorem}
The support $\Supp{\mathcal{L}_S}$ is the set of triples of distributions $(p_A, p_B, p_C)$ such that $\lbrace \eta(\xi_A) - \eta(\xi_B), \eta(\xi_B) - \eta(\xi_A) \rbrace \cap \ImS$ and $\lbrace \eta(\xi_C) - \eta(\xi_A), \eta(\xi_A) - \eta(\xi_C) \rbrace \cap \ImS$ are not empty, and either $\eta(\xi_C) = \eta(\xi_B)$ or $\eta(\xi_C) + \eta(\xi_B) - \eta(\xi_A) \in \eta(\Xi)$. 
\end{theorem}
\begin{proof}
Suppose $(p_A, p_B, p_C)$ satisfies the conditions. If $\eta(\xi_B) = \eta(\xi_C)$, we define $\xi_D = \xi_A$, and otherwise we take $\xi_D$ such that $\eta(\xi_D) = \eta(\xi_C) + \eta(\xi_B) - \eta(\xi_A)$ (which is possible by hypothesis). This defines an analogy by Theorem~\ref{thm:valid-analogies-expfamily}, and consequently $(p_A, p_B, p_C) \in \Supp{\mathcal{L}_S}$. 
Conversely, if $(p_A, p_B, p_C) \in \Supp{\mathcal{L}_S}$, there exists $p_D$ such that the analogy holds. We have our result directly by Theorem~\ref{thm:valid-analogies-expfamily}. 
\end{proof}

%% file: sections/examples-exponential-family.tex

The previous section established a complete characterization of Bayes-based analogies for exponential families. The remaining question is therefore not whether analogies exist, but how rich they are for a given statistical model. According to Theorem~\ref{thm:valid-analogies-expfamily}, this richness is entirely determined by the image of the sufficient statistic of the likelihood. We now illustrate this framework on three classical exponential families. Besides demonstrating the generality of the approach, these examples highlight three qualitatively different situations regarding the expressive power of Bayes-based analogies.

\subsection{Bayes-Based Analogies on Bernoulli and Categorical Distributions}

The categorical distribution is a probability distribution over a finite observation space $\mathcal{O} = \lbrace o_1, \dotsc, o_K \rbrace$, parameterized by the probabilities $\pi_k$ that every observation $o_k$ can take. The space of categorical distributions corresponds to the simplex, i.e. the space of vectors $(\pi_1, \dotsc, \pi_K)$ such that $\pi_k \geq 0$ for all $k$ and $\sum_k \pi_k = 1$. In this paper,we restrict our attention to strictly positive probability vectors, i.e., $\pi_k > 0$ for every $k$, and denote the corresponding parameter space by $\Delta^{K}$. The special case $K = 2$ corresponds to the Bernoulli distribution.

\paragraph{Natural parameters.} A standard choice for the minimal parameters is the vector of log-ratios:
\begin{equation}
    \eta(\xi) =  \left( \log\frac{\pi_1}{\pi_K}, \dotsc, \log\frac{\pi_{K-1}}{\pi_K} \right) \in \mathbb{R}^{K-1}
    \label{eqn:natural-parameters-categorical}
\end{equation}
The corresponding natural statistics $T(\theta)$ are then given by:
\begin{equation}
    T(\theta) = \big( \mathbb{I}(\theta = 1), \dotsc, \mathbb{I}(\theta = K - 1) \big)
    \label{eqn:sufficient-statistics-categorical}
\end{equation}

\paragraph{Compatible likelihood.} Given $K$ distributions $p_1, \dotsc, p_K$, we define the likelihood $\mathcal{L}_{MM}^K$ as a special case of $\mathcal{L}^S$ with:
\begin{equation}
    S(x) = \left( \log \frac{p_1(x)}{p_K(x)}, \dotsc, \log \frac{p_{K-1}(x)}{p_K(x)} \right).
\end{equation}
We note that this distribution corresponds to a mixture model with distributions $p_1, \dotsc, p_K$, and that $\pi_k$ is then the probability of the point to be generated by $p_k$. 


\paragraph{Maximal analogy.} The existence of a maximal analogy follows from the existence of parameters s.t. $\ImS = \mathbb{R}^{K-1}$. 

\begin{proposition}
There exists a maximal Bayes-based proportional analogy between categorical distributions induced by the likelihood
$\mathcal L_{\MM}^K$.
\end{proposition}

\begin{proof}
Let
$p_k=\mathcal N(\mu_k,\Sigma)$,
where $\Sigma$ is positive definite and
$\mu_1-\mu_K,\dots,\mu_{K-1}-\mu_K$
form a basis of $\mathbb R^{K-1}$.
Then $S(x) = A x + b$ for some invertible matrix $A$ and vector $b$, so that $\ImS = \mathbb R^{K-1}$. 
\end{proof}

Note that all the results above apply also to the case of Bernoulli distribution $\Ber(\pi)$, that is a categorical distribution with $K = 2$. In that case, the natural parameter is simply the logit of~$\pi$.

\subsection{Bayes-Based Analogies on Multivariate Normal Distributions}

A multivariate normal distribution $\mathcal{N}(\boldsymbol{\mu}, \boldsymbol{\Sigma})$ is a probability distribution over the parameter space $\Theta = \mathbb{R}^d$ characterized by its mean vector $\boldsymbol{\mu} \in \mathbb{R}^d$ and covariance matrix $\boldsymbol{\Sigma} \in \mathbb{S}_{++}^d$, where $\mathbb{S}_{++}^d$ denotes the set of symmetric positive definite matrices. Its density with respect to $\boldsymbol{\theta}$ is:
\[
p(\boldsymbol{\theta} \mid \boldsymbol{\mu}, \boldsymbol{\Sigma}) = \frac{1}{\sqrt{(2 \pi)^d |\boldsymbol{\Sigma}| }} \exp \left(- \frac{1}{2} (\boldsymbol{\theta} - \boldsymbol{\mu})^T \boldsymbol{\Sigma}^{-1} (\boldsymbol{\theta} - \boldsymbol{\mu}) \right). 
\]

\paragraph{Natural parameters.}
Following our minimal exponential family formalism, a standard choice for the natural parameters is $\eta(\boldsymbol{\mu}, \boldsymbol{\Sigma}) = (\boldsymbol{\Sigma}^{-1} \boldsymbol{\mu}, - \frac{1}{2}\text{vec}(\boldsymbol{\Sigma}^{-1}))$. The corresponding minimal sufficient statistics are $T(\boldsymbol{\theta}) = (\boldsymbol{\theta}, \text{vec}(\boldsymbol{\theta} \boldsymbol{\theta}^T))$. 

\paragraph{Compatible likelihood.}
We define the class $\mathcal{L}^{\mathrm{quad}}_{A,B}$ of likelihoods over observations $x \in \mathcal{X}$ of the form:
\begin{equation}
    p(x \mid \boldsymbol{\theta})
    \propto
    \exp\left(
        A(x)^\top \boldsymbol{\theta}
        -
        \frac12
        \boldsymbol{\theta}^\top B(x) \boldsymbol{\theta}
    \right),
    \label{eqn:likelihood-multivariate}
\end{equation}
where $A(x) \in \mathbb{R}^d$ and $B(x) \in \mathbb{S}_{++}^d$ for every observation $x$. Using the matrix trace identity $\boldsymbol{\theta}^T B(x) \boldsymbol{\theta} = \text{vec}(B(x))^T \text{vec}(\boldsymbol{\theta}\boldsymbol{\theta}^T)$, this family maps exactly to our general framework by setting the data sufficient statistic to $S(x) = (A(x), - \frac{1}{2} \text{vec}(B(x)))^T$. 


\paragraph{Maximal analogy.}

\begin{proposition}
There exists a unique maximal Bayes-based proportional analogy relation induced by likelihoods of the form $\mathcal{L}^{\mathrm{quad}}_{A,B}$, obtained when $\ImS = \mathbb{R}^d \times \text{vec}(\mathbb{S}_{++}^d)$.
\end{proposition}

\begin{proof}
We consider a standard Gaussian linear regression observation model $y = X\boldsymbol{\theta} + \boldsymbol{\varepsilon}$, where $X \in \mathbb{R}^{d \times d}$ is an invertible design matrix and $\boldsymbol{\varepsilon} \sim \mathcal{N}(0, \sigma^2 I_d)$. The complete-data likelihood belongs to $\mathcal{L}^{\mathrm{quad}}_{A,B}$ with:
\[
A(X,y) = \frac{1}{\sigma^2}X^\top y, \qquad B(X) = \frac{1}{\sigma^2}X^\top X.
\]
Let $M \in \mathbb{S}_{++}^d$ and $v \in \mathbb{R}^d$ be an arbitrary destination coordinate pair in the parameter space. We choose the matrix $X = \sigma M^{1/2}$, where $M^{1/2}$ denotes the unique symmetric positive definite square root of $M$. Substituting this yields $B(X) = \frac{1}{\sigma^2}X^\top X = M$. Since $M \in \mathbb{S}_{++}^d$, its square root $M^{1/2}$ is strictly invertible, making $X$ invertible. Choosing the data vector $y = \sigma^2 (X^\top)^{-1}v$ yields $A(X,y) = \frac{1}{\sigma^2}X^\top y = v$. Therefore, every arbitrary parameter pair $(v, -\frac{1}{2}\text{vec}(M))$ is fully realizable by a valid observation tuple $(X, y)$, proving that $\ImS = \mathbb{R}^d \times \text{vec}(\mathbb{S}_{++}^d)$. 
\end{proof}

\subsection{Bayes-Based Analogies on Dirichlet Distributions}

The Dirichlet distribution is a probability distribution of the simplex $\Delta$, characterized by a vector $\alpha = (\alpha_1, \dotsc, \alpha_K)$. Its density with respect to $\boldsymbol{\alpha}$ is:
\[
p(\theta_1, \dotsc, \theta_K \mid \boldsymbol{\alpha}) = \frac{\Gamma\left(\sum_{k=1}^K \alpha_k \right)}{ \prod_{k=1}^K \Gamma(\alpha_k)} \prod_{k=1}^K \theta_k^{\alpha_k - 1}
\]

\paragraph{Natural parameters.} A standard choice of natural parameters is $\eta(\boldsymbol{\alpha}) = (\alpha_1 - 1, \dotsc, \alpha_K - 1) \in (1, +\infty)^K$. The corresponding sufficient statistics are $T(\boldsymbol{\theta}) = (\log \theta_1, \dotsc, \log \theta_K)$.

\paragraph{Compatible likelihood.} With this choice of sufficient statistics, the likelihood family $\mathcal{L}^S$ is the family of distributions of the form:
\begin{equation}
    p(x\mid \theta) = h(x) \prod_{k=1}^K \theta_k^{S_k(x)}
    \label{eqn:compatible-likelihood-dirichlet}
\end{equation}

\begin{proposition}
Let $\mu$ be the reference measure on the observation space, and define the measure $\nu$ by $\nu(dx) = h(x) \mu(dx)$. Every coordinate of $S$ is non-negative almost everywhere with respect to $\nu$.
\label{prop:positive-as-dirichlet-likelihood}
\end{proposition}

\begin{proof}

Consider a compatible likelihood of the form of Equation~\ref{eqn:compatible-likelihood-dirichlet}, where $\theta$ belongs to the interior of the probability simplex. 

Fix a coordinate $j$. We assume, for contradiction, that the set of observations $x$ for which $S_j(x)$ is negative has strictly positive $\nu$-measure.

We define a sequence $\theta(t)$ such that $\theta_j(t) = \exp(-t)$ and $\theta_k(t) = \frac{1- \exp(-t)}{K - 1}$ for $k \neq j$. For any $x$ such that $S_j(x)$, we then have $\theta_j(t)^{S_j(x)} \rightarrow + \infty$ and $\theta_k(t)^{S_h(x)}$ tends to a positive constant for $k \neq j$. Consequently, $\lim_t \prod_{k=1}^K \theta_k(t)^{S_k(x)} = + \infty$. By Fatou's lemma, we then have $\lim_t \int \prod_{k=1}^K \theta_k(t)^{S_k(x)} \nu(dx) = +\infty$, which is impossible for a probability density. This shows that $S_k(x) \geq 0$ $\nu$-almost everywhere for all $k$.
\end{proof}

\paragraph{Maximal analogy.} Based on the following proposition, we can build a likelihood that makes the analogy maximal in the sense that $\ImS = \mathbb{R}^K$. 

\begin{proposition}
There exists a maximal Bayes-based proportional analogy between Dirichlet distributions induced by likelihood $\mathcal{L}^{S+}$ that satisfy Proposition~\ref{prop:positive-as-dirichlet-likelihood}. 
\end{proposition}
\begin{proof}
Let $\mathcal{C} \subset [0, 1]$ be the Cantor set. We remind that $\mathcal{C}$ is an uncountable measurable set of Lebesgue measure zero. $\mathcal{C}$ and $\mathbb{R}^K$ are isomorphic as uncountable standard Borel spaces, so there exists a measurable surjective map $T: \mathcal{C} \rightarrow \mathbb{R}^K$. We then define $S(x) = T(x)$ if $x \in \mathcal{C}$ and $S(x) = 0$ otherwise. By construction, we have $\ImS = \mathbb{R}^K$. Since $\mathcal{C}$ is a Lebesgue measure zero, we have $\int p(x \mid \theta) dx = 1$, so the defined likelihood is valid. Since $S$ is surjective into $\mathbb{R}^K$, it defines a maximal analogy. 
\end{proof}

Although this construction satisfies the present definition of Bayes-based analogies, it reveals a subtle measure-theoretic aspect of the framework. The existence of an analogy depends only on the existence of observations producing the corresponding Bayesian updates, irrespective of whether these observations belong to sets of positive measure. Consequently, maximality may be achieved through observations that are statistically negligible. Whether such exceptional observations should be regarded as genuine generators of analogies is ultimately a design choice. In the present work, we retain the existential definition, leaving the investigation of stronger notions of reachability for future work.

%% file: sections/approximate-solver.tex
The cases explored so far assume that the distributions of interest are parametric and belong to the exponential family. The question of generalizing to arbitrary distributions is important for potential applications in machine learning. In this section, we discuss a sampling-based approach to this problem inspired by Approximate Bayesian Computation~\cite{sunnaaker2013approximate}.

\subsection{Problem Specification}

We consider the setting where the distributions $p_A, p_B$ and $p_C$ are not known in closed form, but only through finite samples $\lbrace \theta_A^1, \dotsc, \theta_A^{n_A}\rbrace, \lbrace \theta_B^1, \dotsc, \theta_B^{n_B} \rbrace$ and $\lbrace \theta_C^1, \dotsc, \theta_C^{n_C} \rbrace$ where $\theta_A^i, \theta_B^i, \theta_C^i \in \Theta$. 
Given a likelihood $\mathcal{L}$, we aim to generate a sample from an unknown distribution $p_D$ such that $p_A : p_B ::^\mathcal{L} p_C : p_D$. 

Since the latent datasets responsible for the Bayesian updates are unknown, we formulate their estimation as the search for two datasets $x, x^\prime \in \bar{\mathcal{X}}$ satisfying the following consistency conditions. Let $\lbrace \hat{\theta}_B^1, \dotsc, \hat{\theta}_B^{m_B} \rbrace$ and $\lbrace \hat{\theta}_C^1, \dotsc, \hat{\theta}_C^{m_C} \rbrace$ denote samples drawn from the posteriors obtained by updating $p_A$ with $x$ and $x^\prime$ respectively. Likewise, let $\lbrace \hat{\theta}_D^{(B), 1}, \dotsc, \hat{\theta}_D^{(B), m_D} \rbrace$ and $\lbrace \hat{\theta}_D^{(C), 1}, \dotsc, \hat{\theta}_D^{(C), m_D} \rbrace$ denote samples drawn respectively from the update of $p_B$ with $x^\prime$ and $p_C$ with $x$. The datasets $x$ and $x^\prime$ are then sought so to satisfy $\lbrace \theta_B^i \rbrace \approx \lbrace \hat{\theta}_B^i \rbrace$, $\lbrace \theta_C^i \rbrace \approx \lbrace \hat{\theta}_C^i \rbrace$ and $\lbrace \hat{\theta}_D^{(B), i} \rbrace \approx \lbrace \hat{\theta}_D^{(C), i} \rbrace$, where $\approx$ denotes similarity between empirical distributions.

Since the analogy is based on the symmetric relation $\leftrightsquigarrow$, the latent datasets $x$ and $x^\prime$ are not restricted to updating $p_A$ into $p_B$ and $p_C$. They may equally correspond to updates in the reverse direction, i.e., from $p_B$ to $p_A$ and/or from $p_C$ to $p_A$. Consequently, the algorithm evaluates all combinations of forward and reverse updates when searching for a consistent analogical completion.

\subsection{Algorithm}

\paragraph{Posterior sampling.}


Following the standard importance sampling framework~\cite{robert2004monte}, posterior distributions are approximated through weighted sampling. Given a prior distribution $p(\theta)$ and a candidate dataset $x$, one repeatedly samples $\theta^* \sim p(\theta)$ and assigns to each sample an importance weight $w(\theta^*)$ proportional to the likelihood $\mathcal{L}(x \mid \theta^*)$. After normalization of the weights, the resulting weighted sample provides an empirical approximation of the posterior distribution $p(\theta \mid x)$, from which posterior samples can be obtained through resampling. This procedure is used to approximate each of the posterior distributions involved in the consistency conditions of the previous section.

\paragraph{Distribution comparison.}

Since the posterior distributions are represented only through finite samples, they cannot be compared analytically. Instead, similarity between two distributions is estimated through a discrepancy measure defined on empirical samples. Several choices are possible, including the Maximum Mean Discrepancy (MMD), the Wasserstein distance, the energy distance, or other metrics between empirical measures. The choice of discrepancy determines the notion of closeness used throughout the algorithm and replaces the exact equality required in the theoretical definition of analogy.

\paragraph{Joint optimization.}

The latent datasets $x$ and $x'$ are searched jointly so as to simultaneously satisfy the three consistency conditions introduced above. Let $d(\cdot,\cdot)$ denote the chosen discrepancy between empirical distributions. The objective is to find $(x,x')$ minimizing
\[
d\!\left(\{\theta_B^i\},\{\hat{\theta}_B^i\}\right)
+
d\!\left(\{\theta_C^i\},\{\hat{\theta}_C^i\}\right)
+
d\!\left(\{\hat{\theta}_D^{(B),i}\},
         \{\hat{\theta}_D^{(C),i}\}\right).
\]
A candidate pair $(x,x')$ is accepted only if each of the three discrepancies is smaller than a prescribed threshold $\varepsilon$. Otherwise, the analogical equation is considered to have no solution for the chosen likelihood model and tolerance. Whenever several candidate pairs satisfy these conditions, they provide alternative approximations of the unknown distribution $p_D$.

\paragraph{Reverse updates.}

The above procedure assumes that the unknown distribution $p_D$ is obtained by updating either $p_B$ or $p_C$. However, because the analogical relation is defined through the symmetric reachability relation $\leftrightsquigarrow$, it is also possible that $p_D$ is instead the \emph{prior} from which both $p_B$ and $p_C$ are obtained after a Bayesian update. In this case, posterior sampling cannot be performed directly. Let $p(\theta)$ denote the known posterior obtained after updating an unknown prior $p_D(\theta)$ with a dataset $x$. By Bayes' rule, $p_D(\theta)\propto\frac{p(\theta)}{\mathcal{L}(x\mid\theta)}.$- Consequently, given samples $\{\theta^i\}_{i=1}^N$ from $p$, an empirical approximation of $p_D$ can be obtained by importance sampling, assigning weights $w_i \propto 1 / \mathcal{L}(x\mid\theta^i)$ followed by a resampling step. This inverse update allows the proposed algorithm to handle all possible orientations of the analogy relation while remaining entirely sampling-based.

\subsection{Empirical Validation}

We evaluate the proposed sampling-based solver on a synthetic benchmark for which the analytical solution of the analogical equation is known. This setting allows us to assess the extent to which the proposed algorithm recovers the exact solution from sampled distributions.

\paragraph{Protocol.}

A total of $N=3121$ independent analogical equations are generated. Each analogy is constructed by first randomly sampling the parameters of a distribution $p_A$ together with two regression observations under the linear regression likelihood. The corresponding posterior distributions $p_B$ and $p_C$ are then obtained analytically through Bayesian updating, while the target distribution $p_D$ is computed exactly using Theorem~\ref{thm:valid-analogies-expfamily}. Consequently, the analytical solution of every analogical equation is known.
For each analogy, only samples drawn from the distributions $p_A$, $p_B$, and $p_C$ are provided to the sampling algorithm. Each distribution is represented by $100$ particles. 

We evaluate two likelihood models: the normal likelihood and the linear regression likelihood. Although both belong to the class $\mathcal{L}_S$, only the latter induces maximal analogies. As discussed above, the normal likelihood generates only a limited support of analogical transformations.

To recover the latent observations, we compare two optimization strategies: an exhaustive grid search and Bayesian optimization~\cite{frazier2018bayesian}. For the normal likelihood, the search space is discretized into a $100\times100$ grid over $\mathcal{X}^2$. For the regression likelihood, each observation is parameterized by a predictor-response pair $(z,y)$. We discretize the predictor into $15$ values and the response into $31$ values, yielding $15\times31=465$ candidate observations and therefore $465^2=216,225$ candidate observation pairs for a single analogical equation.

A candidate solution is retained only after two successive filtering steps. First, every importance-sampling update must satisfy an effective sample size $\mathrm{ESS}= 1 / \sum_i w_i^2$ of at least $15$, ensuring that the posterior approximation is not dominated by only a few particles. Second, the candidate must satisfy the reconstruction constraints by requiring the sum of the three Wasserstein distances to remain below $0.5$. This ensures both that the reconstructed distributions accurately approximate $p_B$ and $p_C$, and that the two independent reconstructions of $p_D$ are mutually consistent.

\paragraph{Metrics.}

The quality of the importance-sampling approximation is assessed through the effective sample size (ESS). The accuracy of the recovered solution is primarily evaluated using the $2$-Wasserstein distance between the analytical Gaussian solution and the Gaussian fitted to the recovered particles. We additionally report the absolute errors on the recovered mean and variance, $(\hat{\mu}_D,\hat{\sigma}_D^2)$, with respect to the analytical solution. Since $p_D$ is reconstructed through two independent analogy paths, we conservatively report the minimum ESS of the two reconstructions.

\paragraph{Implementation and execution.}
The source code is available on the Github repository \url{https://github.com/ppaamm/Analogies-on-Probability-Distributions-by-Bayes-Updating}. 
All experiments were performed on a laptop equipped with a 13th Gen Intel(R) Core(TM) i7-13700H CPU and 32 GB of RAM. No GPU acceleration was used.

\paragraph{Empirical results.}

For the normal loss, the algorithm does not recover any solution. This result was expected, since the support of the analogy is of measure zero in this case. This observation indicates that our algorithm does not retrieve incorrect solutions when the analogy is not maximal. 

The grid search with the regression loss was performed only on a few analogies, because of very high computation times. Whereas a typical grid search lasted between $5$ and $20$ seconds for the search with the normal likelihood (grid of size $10,000$), it lasted up to $10$ minutes with the regression likelihood (grid of size $216,225$). This encourages the use of Bayesian optimization for a faster search. The results presented below were all obtained used Bayesian Optimization.

The proposed solver successfully recovers solutions for $1918$ out of the $3121$ equations, i.e. $61.5\%$ of the generated analogies. A qualitative analysis indicates that this percentage depends mostly on the number of iterations in the BO, limited here to $40$. 

\begin{figure*}[t]
    \centering

    \begin{subfigure}[t]{0.32\textwidth}
        \centering
        \includegraphics[width=\linewidth]{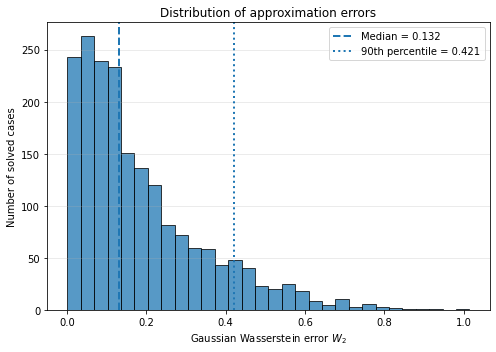}
        \caption{Distribution of the Wasserstein distance.}
        \label{fig:wasserstein_hist}
    \end{subfigure}
    \hfill
    \begin{subfigure}[t]{0.32\textwidth}
        \centering
        \includegraphics[width=.75\linewidth]{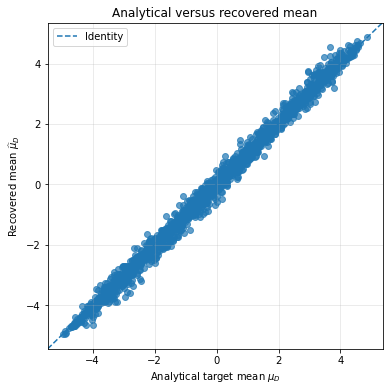}
        \caption{Correlation between $\mu_D$ and estimated $\hat{\mu}_D$.}
        \label{fig:correlation-mean}
    \end{subfigure}
    \hfill
    \begin{subfigure}[t]{0.32\textwidth}
        \centering
        \includegraphics[width=.75\linewidth]{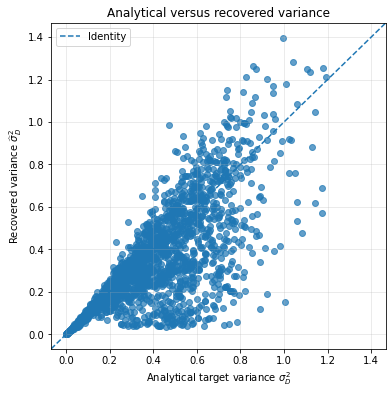}
        \caption{Correlation between $\sigma_D^2$ and estimated $\hat{\sigma}_D^2$.}
        \label{fig:correlation-variance}
    \end{subfigure}

    \caption{Empirical evaluation of the sampled solutions. (a) Distribution of the Wasserstein distance between the analytical and recovered solutions. (b) Distribution of the effective sample size. (c) Relationship between the minimum effective sample size and the reconstruction error.}
    \label{fig:empirical_results}
\end{figure*}

Among successful runs, the approximation error remains generally small, with a median Gaussian Wasserstein distance of $0.132$ and a $90$th percentile of $0.421$ (Figure~\ref{fig:wasserstein_hist}). The error distribution is strongly right-skewed, indicating that most analogies are solved accurately while a relatively small number of difficult instances produce substantially larger errors. An analysis of the correlation between the predicted parameters and the parameters of the analytical solutions (Figures~\ref{fig:correlation-mean} and \ref{fig:correlation-variance}) shows that the error is mostly due to the variance that tends to be under-estimated. We observe on contrary a strong correlation in the means, with a Pearson correlation of $0.997$. 

\begin{figure}
    \centering
    \includegraphics[width=0.9\linewidth]{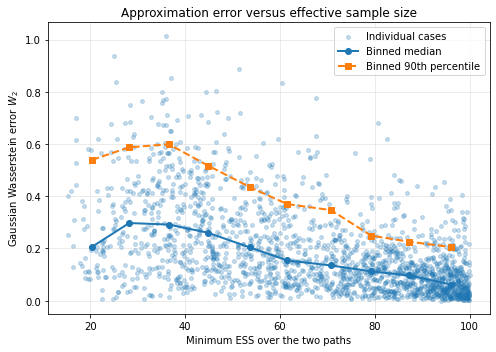}
    \caption{Comparison of the approximation error and the effective sample size. It can be observed that the approximation error tends to decrease when the ESS increases, which indicates that particle degeneracy is one of the main limitations of the current implementation.}
    \label{fig:ess_vs_error}
\end{figure}

Figure~\ref{fig:ess_vs_error} investigates the influence of the effective sample size on the approximation quality. A clear decreasing trend is observed: both the median approximation error and its upper quantiles decrease as ESS increases. In particular, large approximation errors occur almost exclusively for small ESS values, whereas high ESS values consistently lead to accurate reconstructions. Nevertheless, accurate solutions can occasionally be obtained even for relatively small ESS, suggesting that ESS should primarily be interpreted as a measure of numerical reliability rather than a deterministic predictor of approximation accuracy.

Overall, these experiments provide empirical evidence that the proposed sampling algorithm accurately approximates analogical solutions whenever the importance sampling step remains sufficiently effective. They further indicate that our current implementation has two main practical limitations: insufficient search misses solutions, and particle degeneracy globally produces higher approximation errors.

%% file: sections/conclusion.tex
In this paper, we introduced a characterization of analogies on probability distributions that satisfies the predicates of proportional analogy while relying on the intrinsic properties of probability distributions. Our definition yields a simple characterization of analogies between members of an exponential family, expressed as an arithmetic proportion between their natural parameters, up to an inversion property on a set of problems of measure zero. Unlike existing approaches, our framework also naturally extends to non-parametric families. To this end, we proposed a sampling-based algorithm that approximately solves analogical equations from samples of the first three distributions. A proof-of-concept experimental evaluation demonstrated the feasibility of this approach, although the current implementation remains limited by the approximation quality of importance sampling. Future work will investigate more advanced inference techniques, such as posterior sampling and Approximate Bayesian Computation~\cite{sunnaaker2013approximate}, to improve both the robustness and scalability of the solver.

More broadly, we believe that this work provides a principled foundation for connecting analogical reasoning with machine learning. Beyond its theoretical interest, the proposed framework opens new perspectives for applications such as transfer learning, data augmentation, distribution reconstruction, case-based reasoning, and meta-learning, where reasoning directly over probability distributions may provide a natural and expressive mechanism for knowledge transfer.